\documentclass[journal]{IEEEtran}
\usepackage[compress]{cite}
\usepackage{amsmath,amsfonts}
\usepackage{amsthm}
\usepackage{url,xspace}
\usepackage{xcolor}
\usepackage{graphicx}
\usepackage[ruled,vlined]{algorithm2e}
\usepackage{etoolbox}
\usepackage{nicematrix} 

\newcommand{\Change}[1]{#1}

\makeatletter 
\patchcmd{\@algocf@start}
  {-1.5em}
  {0pt}
  {}{}
\makeatother

\newtheorem{proposition}{Proposition}

\newcommand{\V}[1]{\boldsymbol{#1}}
\newcommand{\M}[1]{\mathbf{#1}}

\newcommand{\algoname}{MuChaPro\xspace}

\begin{document}

\title{Just Project!\\ Multi-Channel Despeckling, the Easy Way}

\author{Loïc Denis,~\IEEEmembership{Senior Member,~IEEE},
\thanks{Loïc Denis is with Université Jean Monnet Saint-Etienne, CNRS, Institut d Optique Graduate School, Laboratoire Hubert Curien UMR 5516, F-42023, SAINT-ETIENNE, France and also with LTCI, Télécom Paris, Institut Polytechnique de Paris, Palaiseau, France, (e-mail: loic.denis@univ-st-etienne.fr).}
Emanuele~Dalsasso, \thanks{Emanuele Dalsasso is with Environmental Computational Science and Earth Observation (ECEO) laboratory, EPFL, 1951 Sion Switzerland.}
Florence~Tupin,~\IEEEmembership{Senior Member,~IEEE}%
\thanks{Florence Tupin is with 
LTCI, Télécom Paris, Institut Polytechnique de Paris, Palaiseau, France.}
}

\markboth{}{}%

\maketitle

\begin{abstract}
Reducing speckle fluctuations in multi-channel SAR images is essential in many applications
of SAR imaging such as polarimetric classification or interferometric height estimation. While single-channel
despeckling has widely benefited from the application of deep learning techniques, extensions
to multi-channel SAR images are much more challenging.

This paper introduces \algoname, a generic framework that exploits existing single-channel
despeckling methods. The key idea is to generate numerous single-channel projections, restore
these projections, and recombine them into the final multi-channel estimate. This simple
approach is shown to be effective in polarimetric and/or interferometric modalities. A special
appeal of \algoname is the possibility to apply a self-supervised training strategy to learn
sensor-specific networks for single-channel despeckling.

\end{abstract}


\IEEEpeerreviewmaketitle

\section{Introduction}
Synthetic aperture radar (SAR) imaging is an invaluable technique for Earth observation
due to its unique ability to see through clouds and its sensitivity to surface
roughness and soil moisture. Beyond providing an image of the back-scattered echo
intensities, the polarimetric and interferometric modes offer additional information
about the scattering mechanisms (single, double, or triple bounces, surface or volume
scattering), the topography (elevation, displacement), or the 3D location of scatterers.
All these features are achieved thanks to coherent measurement and processing of the
SAR signals. Due to the use of coherent illumination, the speckle phenomenon arises:
constructive or destructive interferences occur within each resolution cell due to the
coherent summation of several back-scattered echoes. Speckle manifests itself in SAR images
in the form of strong fluctuations of the intensity. In multi-channel SAR imaging, it
corrupts the interferometric phase and the polarimetric covariance matrix, making the
analysis of these images and their automatic processing challenging.

Speckle fluctuations can be reduced by averaging pixel values within a small window.
To reach a satisfying amount of smoothing, tens of pixels must be averaged, which implies
a severe blurring of the image structures as well as notable errors when mixing, within 
a given window, scatterers with very different backscattering power
\cite{formont2010statistical}. More refined filtering strategies are required to reduce
speckle fluctuations while preserving the spatial resolution. Many different approaches
have been imagined: some based on the selection of pixels within homogeneous oriented windows
\cite{lee1999polarimetric}, clustered by region-growing \cite{vasile2006intensity},
identified based on patch similarity \cite{chen2010nonlocal,deledalle2014nl,sica2018insar};
other approaches follow a variational approach \cite{hongxing2014interferometric, 
nie2016nonlocal,deledalle2017mulog,deledalle2018very} \Change{or consider Beltrami flows to operate the filtering operation on manifolds \cite{amao2019beltrami}}. More recently, deep neural networks
have led to very successful despeckling techniques \cite{fracastoro2021deep,rasti2021image, zhu2021deep}. 

The standard way to train a deep neural network is through supervised learning, i.e., by
training on pairs formed by a speckled image (provided as input to the network) and the
corresponding speckle-free image (the expected output of the network). Generating a
training set with such image pairs is challenging, in particular for multi-channel SAR
imaging. To obtain the speckle-free image corresponding to an image corrupted by speckle,
the easiest way is to start with a speckle-free image and simulate synthetic speckle.
Speckle-free images can be obtained by degrading the resolution of very-high-resolution
images (i.e., by spatial multi-looking) or by averaging long time series (i.e., by
temporal multi-looking). While temporal multi-looking may be applied to SAR polarimetry (PolSAR),
it generally does not work in SAR interferometry (InSAR): the availability of time series
of interferograms with a constant baseline and negligible coherence between interferograms
is generally not a realistic scenario. Simulating speckle while accounting for the spatial
correlations due to the SAR transfer function (zero-padding and spectral apodization) and the
spatial and temporal decorrelations in InSAR can be challenging. Yet, it is necessary in order
to obtain networks that perform well on actual SAR data. 
The difficulty of producing training sets for supervised learning increases with the
dimensionality of multi-channel SAR images, representing a real limit for multi-baseline
InSAR, PolInSAR or SAR tomography applications.

\begin{figure*}[!t]
    \includegraphics[width=\textwidth]{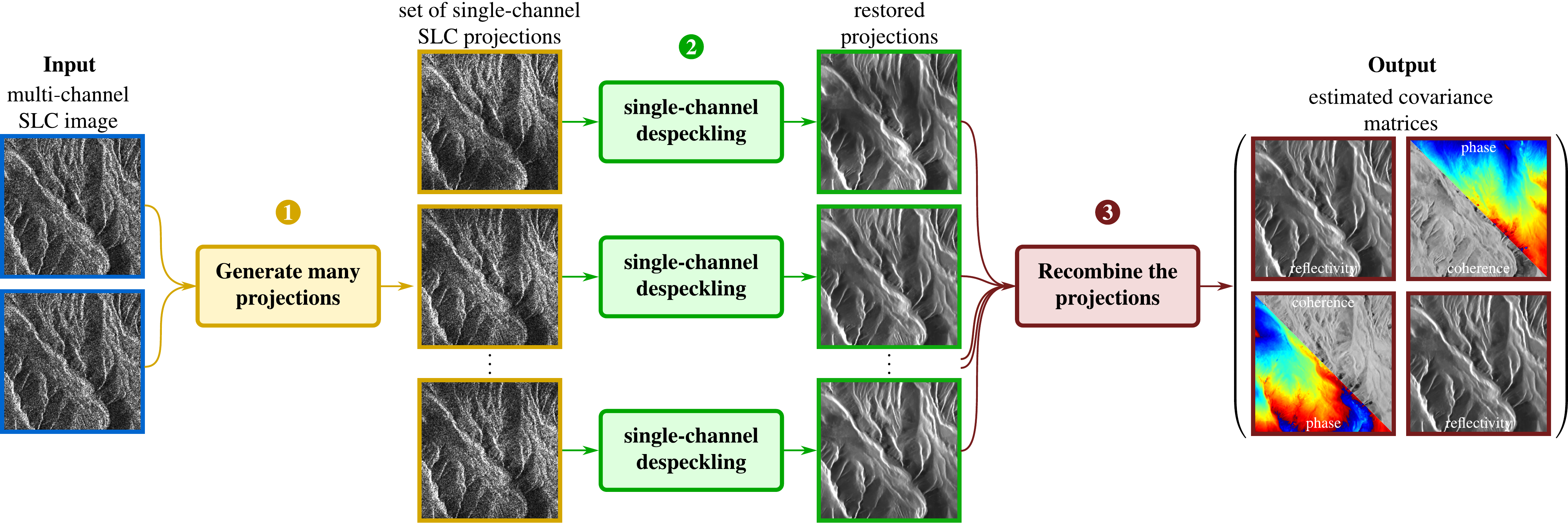}
    \caption{Scheme of the \algoname method introduced in this paper, illustrated here on SAR interferometry: multi-channel SAR images are projected
    into single-channel SLC images which are, after despeckling, recombined to produce the final estimate.}
    \label{fig:ppe}
\end{figure*}

To circumvent the difficulty of building training sets for supervised learning, self-supervised
approaches have been proposed for SAR despeckling \cite{dalsasso2022self}. Self-supervised
despeckling strategies share a common principle: the splitting of the noisy observations into two
subsets, one provided to the despeckling network, the other only available to evaluate the 
training loss function. Methods differ in the way this splitting is performed:
\begin{enumerate}
    \item In SAR2SAR \cite{dalsasso2021sar2sar}, images captured at different
    dates are used: a reference date is processed by the network and compared, during training,
    to a second noisy date. To account for possible changes occurring between these two dates,
    a change-compensation step is added.
    \item Blind-spot approaches \cite{laine2019high} consist 
    of designing a specific network architecture such that the receptive field (i.e., the
    area in the input image used to predict the value at a given pixel of the output image)
    excludes the central pixel. The noisy value of that pixel can then be used to evaluate
    a training loss (the log-likelihood, conditioned by the neighborhood made available to the
    network). In Speckle2Void \cite{molini2021speckle2void}, Molini \emph{et al.} extend
    this strategy to SAR despeckling. An alternative is to use more conventional network
    architectures and apply a preprocessing step to mask out some pixels and use their
    value to evaluate the training loss, as suggested in \cite{sica2024cohessl}.
    \item MERLIN \cite{dalsasso2021if} uses the real and imaginary decomposition of a
    single-look complex SAR image and the property that the speckle in each image of this
    decomposition is statistically independent of one another (under some assumptions
    on the SAR transfer function, see \cite{dalsasso2023self}).
\end{enumerate}

A radically different approach is the plug-and-play framework MuLoG \cite{deledalle2017mulog,deledalle2022speckle},\Change{\cite{mendes2024robustness}}.
Derived from a variational framework and the Alternating Directions Method of Multipliers (ADMM)
\cite{boyd2011distributed}, it consists first of decomposing a multi-channel SAR image into
real-valued channels after a matrix-logarithm is applied to the noisy covariance matrices that
capture the polarimetric and/or interferometric information at each pixel, then alternatingly
denoising these real-valued channels separately with a conventional graylevel image denoising
algorithm (based on a deep neural network or not) and recombining all channels within a 
non-linear operation. MuLoG presents the advantage of being applicable to SAR images of
various dimensionality and to readily include pre-trained networks capable of denoising images
corrupted by additive white Gaussian noise. The drawback is the lack of specialization to
a specific SAR sensor. In particular, it does not account for the spatial correlations of
speckle. Since it is based on the matrix-log decomposition of covariance matrices, some amount
of spatial smoothing is necessary, in particular when the dimensionality of the images
increases (in multi-baseline interferometry and SAR tomography).

Deep neural networks have been designed specifically to estimate the
phase and coherence of an interferometric pair. The $\Phi$-Net \cite{sica2021net}
is trained in a supervised way. It applies a decorrelation of the real and imaginary
components of the complex interferogram to form the two channels provided as input to
the network. InSAR-MONet \cite{vitale2022insar} uses a multi-term loss function to
restore the interferometric phase. Digital elevation models and empirical coherence
maps were used to generate pairs of simulated noisy/noise-free phase images. The loss
combines spatial terms, to ensure that the restored phase is close to the ground-truth
phase, and a statistical term to enforce that the estimated noise component follows the
expected distribution.

Specific networks have also been proposed for SAR polarimetry, generally following a
supervised training methodology. In \cite{DL_polsar_despeck1}, the same matrix-logarithm
transform as in MuLoG is applied first, then complex-valued operations are applied
throughout the network (i.e., the real and imaginary parts are not extracted and
processed like a two-channels image but rather complex-convolutions, complex-activations,
 and complex-batch normalizations are applied). Tucker and Potter \cite{DL_polsar_despeck2}
 also apply a log-transform but then perform a real-imaginary decomposition and apply
 a real-valued residual network to estimate the log-transformed speckle component.
A different approach is followed in \cite{lin2023residual} where a network is trained,
in a supervised way, to produce weights leading to a weighted combination of the
neighboring pixels as close as possible to a ground truth.
 By considering pairs of polarimetric images acquired at close dates, a
 network is trained without reference in \cite{li2024sentinel}. Unlike SAR2SAR
 \cite{dalsasso2021sar2sar} \Change{and its recent extension to polarimetric images PolSAR2PolSAR \cite{mendes2024polsar2polsar}}, no compensation for changes is applied, which could represent a 
 serious limitation in quickly evolving areas (e.g., agricultural fields, changes of 
 soil moisture). \Change{To avoid issues with changes and obtain speckle-free covariance matrix estimates, reference \cite{lu2023training} combines temporal averaging and the generation of synthetically-corrupted images to produce a suitable training set. The approach proposed in \cite{mestre2024deep} also relies on temporal averaging but selects patches without change (detected using an Omnibus Test) rather than generating patches with synthetic speckle.}

\emph{Our contributions:} This paper introduces a radically different approach to InSAR 
and PolSAR despeckling. Although the task requires estimating a multi-dimensional and 
complex-valued covariance matrix at each pixel, we suggest reducing the problem to a 
series of single-channel real-valued image despeckling by working on projections, see
Figure \ref{fig:ppe}.
Such an approach, named ``MuChaPro'' in the following, presents several advantages:
\begin{enumerate}
    \item it greatly simplifies the application of deep neural networks for the estimation
    of polarimetric and~/ or interferometric properties,
    \item a network trained for a given sensor~/ resolution can be readily applied to
    various polarimetric or interferometric configurations,
    \item it prevents some of the issues occurring with increasing covariance matrix
    dimensions, in particular, the augmentation of the network complexity required by
    joint processing of numerous channels and the substantial risk of generalization
    issues (the spatial polarimetric/interferometric patterns become more diverse
    and the gap between training and inference steps increases with the dimensionality),
    \item the proposed approach is generic: it can accommodate any despeckling algorithm,
    including deep neural networks of various architectures as well as algorithms that
    do not use neural networks,
    \item various despeckling algorithms can easily be applied in parallel 
    to compare their outputs and get a better sense of possible artifacts / network
    hallucinations,
    \item our approach provides a way to train networks for polarimetric or interferometric 
    despeckling in a self-supervised way.
\end{enumerate}
All those features come at a cost: by reducing the multichannel despeckling problem to a 
series of independent single-channel despeckling tasks, weakly contrasted geometrical
features do not benefit from the reinforcement brought by joint processing. We show in
the following that the capacity of most recent self-supervised networks to restore 
details in single-channel SAR images mitigates this limitation and makes our approach
competitive.

\section{\algoname: Estimating covariance matrices from projections}
\label{sec:muchapro}

\subsection{A generic framework for multi-channel restoration using
single-channel despeckling techniques}
\label{sec:muchapro:ppe}

In order to obtain a method that readily generalizes to multi-channel SAR images with
an arbitrary number of channels, we propose to project these images into a set of
single-channel images and to despeckle these projections. After despeckling, these
single-channel images can be recombined to form the final multi-channel covariance
estimates. Figure \ref{fig:ppe} summarizes the principle of our method. Since
{\bf Mu}lti-{\bf Cha}nnel {\bf Pro}jections are at the core of the framework and
we perform many projections (``{\bf mucha pro}yecci\'on", in Spanish), we call our
method \algoname.

To explain the rationale behind this approach, it is necessary to recall Goodman's 
fully-developed speckle model \cite{Goodm}. A single-look complex (SLC) multi-channel
SAR image $\V z$
contains at each pixel $\ell$ a vector $[\V z]_\ell\in\mathbb{C}^D$ of $D$ complex
amplitudes. Due to speckle, these complex amplitudes are distributed\footnote{we
neglect here the spatial correlations of speckle due to the SAR transfer function,
in section \ref{sec:muchapro:theo} we discuss the impact of these correlations}
according to a complex circular Gaussian distribution $\mathcal{N}_c(\M C_\ell)$
parameterized by the covariance matrix $\M C_\ell\in\mathbb{C}^{D\times D}$ that contains
the polarimetric and/or interferometric information at pixel $\ell$.

Let $\{\V p_k\}_{k=1..K}$ be a set of $K$ vectors of $\mathbb{C}^D$. The complex-valued
images $\{\V s_k\}_{k=1..K}$ obtained by projection of the multi-channel image $\V z$
onto the corresponding vectors $\V p_k$ are defined by:
\begin{align}
    \forall \ell,\,[\V s_k]_\ell=\V p_k^\dagger \cdot [\V z]_\ell\,,
    \label{eq:projz}
\end{align}
where the notation $^\dagger$ indicates the conjuguate-transpose operation. The value
$[\V s_k]_\ell$ is a complex number, i.e., the image $\V s_k$ corresponds to a
single-channel SLC image. Since the projection $[\V s_k]_\ell$ at pixel $\ell$ corresponds
to a linear transform of $[\V z]_\ell$, it also follows a complex circular Gaussian
distribution, with a variance equal to
$\V p_k^\dagger\cdot\M C_\ell\cdot\V p_k\equiv [\V v_k]_\ell\in \mathbb{R}^+$. The SLC
projections can be seen as regular SAR images, corrupted by single-look speckle, and
with an underlying reflectivity $\V v_k$. The application of a single-channel
despeckling algorithm separately on each of these images produces estimates
$\widehat{\V v}_k$ of the true variance $\V v_k$.

The aim is to recover at each pixel $\ell$ the $D\times D$ covariance matrix $\M C_\ell$.
The definition $[\V v_k]_\ell=\V p_k^\dagger\cdot\M C_\ell\cdot\V p_k$ shows that the variances
$[\V v_k]_\ell$ can themselves be seen as projections of the covariance matrices
$\M C_\ell$:
\begin{align}
    \forall k,\,\forall\ell,\,[\V v_k]_\ell=\V q_k^\dagger \cdot [\V c]_\ell\,,
    \label{eq:covproj}
\end{align}
with $[\V c]_\ell=\text{vec}(\M C_\ell)$ the vectorized form of $\M C_\ell$ (i.e., the values
of matrix $\M C_\ell$ are rearranged 
to form a column vector of
dimension $D^2$) and vector $\V q_k\in\mathbb{C}^{D^2}$ is defined by 
$\V q_k=\text{vec}(\V p_k\cdot\V p_k^\dagger)$. In other words, by despeckling the
projections $\V z_k$, we recovered \emph{projections of the covariances} $\M C_\ell$. It
remains to invert these projections, a step that is discussed in the next paragraph.

\subsection{Recovering the covariance matrix from its projections}

The simplest and fastest way to recover, at each pixel $\ell$, the covariance
matrix $\M C_\ell$ consists of solving the following linear least-squares problem:
\begin{align}
    \widehat{\M C}_\ell \in \mathop{\text{arg min}}_{\M C_\ell} \sum_{k=1}^K
    \left(\V p_k^\dagger\cdot\M C_\ell\cdot\V p_k-[\widehat{\V v}_k]_\ell\right)^{\!2}.
    \label{eq:argmin1}
\end{align}
Note that the inversion of the projections is independent from one pixel to the next and
can thus be conducted in parallel.
Using the notations introduced in equation (\ref{eq:covproj}), the least-squares problem
can be rewritten in a more conventional form:
\begin{align}
    [\widehat{\V c}]_\ell \in \mathop{\text{arg min}}_{[\V c]_\ell} \left\|
    \M Q^\dagger\cdot[\V c]_\ell-[\widehat{\V v}]_\ell\right\|_2^2,
    \label{eq:argmin2}
\end{align}
with $\M Q$ the $D^2\times K$ matrix with the $k$-th column equal to $\V q_k$ and
$[\widehat{\V v}]_\ell$ the column vector of the $K$ restored reflectivities at
pixel $\ell$ obtained by despeckling the $K$ projections $\{\V s_k\}_{k=1..K}$.

Provided that there are at least $D^2$ linearly independent projection directions
$\V p_k$, the matrix $\M Q\cdot\M Q^\dagger$ is invertible and the least squares
solution is unique and given by:
\begin{align}
    \forall \ell,\,\widehat{\M C}_\ell^{(\text{\algoname})}
    =\text{vec}^{-1}\!\left(\left(\M Q\cdot\M Q^\dagger\right)^{-1}
    \cdot\M Q\cdot[\widehat{\V v}]_\ell\right),
    \label{eq:muchaproformula}
\end{align}
with $\text{vec}^{-1}$ the reshaping operation that transforms a column vector of
size $D^2$ back to a $D\times D$ matrix.

\newcommand{\algocomment}[1]{\hfill\emph{\small #1}}
\begin{algorithm}[t]
\SetAlgoLined
\SetKwInOut{Input}{input}
\SetKwInOut{Output}{output}
\Input{an $N$-pixels $D$-channels SLC image provided as a matrix $\M Z\in\mathbb{C}^{D\times N}$}
\Input{a single-channel despeckling function $f:\;\mathbb{C}^N\to\mathbb{R}_+^N$ }
\Output{the restored covariance matrices stored as an $D^2 \times N$ matrix
$\widehat{\M C}$
such that the $\ell$-th column contains the $D^2$ values of
$\text{vec}\bigl(\widehat{\M C}_\ell^{(\text{\algoname})}\bigr)$, i.e.,
the covariance matrix
at pixel $\ell$ in vectorized form.}
\BlankLine
\BlankLine
{\bf Step 1}: \emph{generate $K$ SLC images by projection}\\
\nl\label{algline:drawdirections}generate a matrix $\M P\in\mathbb{C}^{D\times K}$
whose columns define $K$ projection directions $\{\V p_k\}_{k=1..K}$\\
\nl\label{algline:project} $\M S\gets\M P^\dagger\cdot \M Z$ \algocomment{(compute the $K$ projections)}\\
\BlankLine
{\bf Step 2}: \emph{despeckle the projections}\\
\For{$k=1$ {\rm to} $K$}{
    \nl    $\widehat{\M V}_{k,\bullet}\gets f(\M S_{k,\bullet})$
    \algocomment{(despeckle $k$-th single-channel image)}
}
\BlankLine
{\bf Step 3}: \emph{recombine restored projections}\\
\nl build matrix $\M Q\in\mathbb{R}^{D^2\times K}$ such that
$\M Q_{\bullet,k}=\text{vec}(\M P_{\bullet,k}^{\phantom{\dagger}}\M P_{\bullet,k}^\dagger)$\\
\nl\label{algline:invertproj}
$\widehat{\M C}\gets \left(\M Q\cdot\M Q^\dagger\right)^{-1}\!\!\cdot\M Q^{\dagger}\cdot\widehat{\M V}
$\algocomment{(recover all covariance matrices)}\\ 
\caption{{\bf \algoname}, Multi-channel despeckling}
 \label{alg:MuChaPro}
\end{algorithm}




The algorithm \algoname given at the top of page \pageref{alg:MuChaPro} summarizes the 3 steps
of the proposed method by using matrix notation to define each step using linear algebra for an
efficient implementation in languages such as Python or Matlab.

\medskip
The formulation of the least-squares problem (\ref{eq:argmin2}) does not exploit the Hermitian
symmetry of the matrix $\M C_\ell$. Due to this symmetry, there are not $D^2$ \emph{complex}
unknowns but actually $D^2$ \emph{real} unknowns:
\begin{align}
    \M C_\ell =
    \begin{pmatrix}
        \text{C}_{11}   & \text{C}_{12} & \cdots & \text{C}_{1D}\\
        \text{C}_{12}^* & \ddots &\ddots &\vdots\\
        \vdots   & \ddots & \ddots & \text{C}_{D-1\;D}\\
        \text{C}_{1D}^* & \cdots & \text{C}_{D-1\;D}^* & \text{C}_{DD}
    \end{pmatrix}
\end{align}
so the vector of all real-valued unknowns $[\V c]_\ell$ at pixel $\ell$ can be structured
as follows:
\begin{multline}
    [\V c]_\ell = 
    \biggl(
        \text{C}_{11}\;\cdots\; \text{C}_{DD} \;\;
        \mathbb{R}\text{e}\bigl(\text{C}_{12}\bigr)\;\cdots\;
        \mathbb{R}\text{e}\bigl(\text{C}_{D-1\;D}\bigr)\\
        \mathbb{I}\text{m}\bigl(\text{C}_{12}\bigr)\;\cdots\;
        \mathbb{I}\text{m}\bigl(\text{C}_{D-1\;D}\bigr)
    \biggr)^{\!\text{t}}
\end{multline}
and the corresponding matrix $\M Q$ in equation (\ref{eq:argmin2}) is structured into 3
blocs:
\begin{align}
    \M Q =
    \begin{pNiceMatrix}[c,left-margin=0.6em,right-margin=0.6em,cell-space-limits=3pt]
        \Block[draw]{3-3}{}
        \mid    &   & \mid    \\
        |\V p_1|^2 & \cdots   & |\V p_K|^2 \\
        \mid   &    & \mid    \\ \\
        \Block[draw]{3-3}{}
        \mid    &   & \mid    \\
        2\mathbb{R}\text{e}(\M U\V q_1) & \cdots   & 2\mathbb{R}\text{e}(\M U\V q_K) \\
        \mid   &    & \mid    \\ \\
        \Block[draw]{3-3}{}
        \mid    &   & \mid    \\
        -2\mathbb{I}\text{m}(\M U\V q_1) & \cdots   & -2\mathbb{I}\text{m}(\M U\V q_K) \\
        \mid   &    & \mid    \\
    \end{pNiceMatrix},
    \label{eq:matQH}
\end{align}
where $\M U$ is a $D(D-1)\times D^2$ matrix obtained from the identity matrix of size $D^2\times D^2$
by keeping only the rows corresponding to elements above the diagonal, i.e., corresponding
to the upper triangular part of the $D\times D$ matrix $\V p_k\cdot \V p_k^\dagger$.

The expression of the least-squares solution remains unchanged provided that this new
definition for matrix $\M Q$ is used.
As discussed in paragraph \ref{sec:muchapro:selproj}, this parameterization
using a real-valued vector $\V c$ leads to better-conditioned matrices $\M Q\cdot\M Q^\dagger$
and more stable estimates. Note that the least squares solution does not guarantee that
the reconstructed matrix
$\widehat{\M C}_\ell^{(\text{\algoname})}$ is positive definite. 
A post-processing
step can be applied to clip negative reflectivity estimates and coherence values
greater or equal to 1, see algorithm \ref{alg:forcepd}.

\begin{algorithm}[t]
\SetAlgoLined
\SetKwInOut{Input}{input}
\SetKwInOut{Output}{output}
\Input{a $D\times D$ symmetrical matrix $\M C$}
\Input{thermal noise equivalent reflectivity $R_{\text{thml}}$}
\Input{maximum coherence value $\rho_{\text{max}}<1$}
\Output{a $D\times D$ positive definite matrix $\bar{\M C}$}
\BlankLine
\BlankLine
{\bf Step 1}: \emph{enforce reflectivities at least equal to $R_{\text{thml}}$}\\
\For{$d=1$ {\rm to} $D$}{
    \nl    $\bigl[\bar{\M C}\bigr]_{d,d}\gets \max(R_{\text{thml}},\bigl[\M C\bigr]_{d,d})$
    \algocomment{(clip low/negative values)}
}
\BlankLine
{\bf Step 2}: \emph{enforce coherence values below $\rho_{\text{max}}$}\\
\For{$d_1=1$ {\rm to} $D$}{
    \For{$d_2=d_1+1$ {\rm to} $D$}{
        \nl $\rho^{-1}\gets
        \max\!\left(\rho_{\text{max}},\frac{\bigl|\bigl[\M C\bigr]_{d_1,d_2}\bigr|}{\sqrt{\bigl[\bar{\M C}\bigr]_{d_1,d_1}\bigl[\bar{\M C}\bigr]_{d_2,d_2}}}\right)^{\!-1}$\\
    \nl    $\bigl[\bar{\M C}\bigr]_{d_1,d_2}\gets \rho_{\text{max}}\cdot\rho^{-1}\cdot \bigl[{\M C}\bigr]_{d_1,d_2}$\\
    \nl    $\bigl[\bar{\M C}\bigr]_{d_2,d_1}\gets \rho_{\text{max}}\cdot\rho^{-1}\cdot \bigl[{\M C}\bigr]_{d_2,d_1}$
    }
}
\caption{Enforcing positive-definiteness}
 \label{alg:forcepd}
\end{algorithm}

\subsection{Theoretical properties of \algoname}
\label{sec:muchapro:theo}

In this paragraph we establish two theoretical results that have strong implications for
the application of \algoname to multi-channel SAR data. The first one indicates that, for
the class of linear filters, multi-channel filtering and \algoname are equivalent. There
is no loss of statistical efficiency when processing single-channel projections rather
than performing the standard multi-channel linear processing. The second one shows that
the real and imaginary parts of the single-look-complex projections are statistically
independent and, therefore, lend themselves to the application of the self-supervised
learning method MERLIN \cite{dalsasso2021if}, in contrast to the case of multi-channel
SAR images for which the real and imaginary parts are in general correlated \cite{10497596}. 

\begin{proposition} \label{prop:equivlin}
    Let $\widehat{\M C}^{(\text{\rm lin})}$ be a linear multi-channel estimator defined by:
    $$\widehat{\M C}^{(\text{\rm lin})}=\sum_{\ell=1}^L w_\ell \cdot [\V z]_\ell
                                        \cdot [\V z]_\ell^\dagger\,,$$
    where $[\V z]_\ell\in\mathbb{C}^D$ is the vector of complex amplitudes at pixel $\ell$
    of the multi-channel SLC image $\V z$.

    Let $\bigl\{\widehat{v}_k^{(\text{\rm lin})}\bigr\}_{k=1..K}$ be the corresponding linear single-channel
    estimators defined by:
    $$\forall k,\,1\leq k \leq K,\;\widehat{v}_k^{(\text{\rm lin})}
    =\sum_{\ell=1}^L w_\ell \cdot |[\V s_k]_\ell|^2$$
    with the same weights $w_\ell$ and $\{\V s_k\}_{k=1..K}$ the $K$ projections of $\V z$
    defined in equation (\ref{eq:projz}).

    If the set of projections $\{\V p_k\}_{k=1..K}$ is such that the matrix $\M Q\cdot\M Q^\dagger$
    in equation (\ref{eq:muchaproformula}) is invertible, then the estimators
    $\widehat{\M C}^{(\text{\rm lin})}$ and $\widehat{\M C}^{(\text{\rm \algoname})}$ applied to
    the linearly despeckled reflectivities $\bigl\{\widehat{\V v}_k^{(\text{\rm lin})}\bigr\}_{k=1..K}$
    coincide.
\end{proposition}
\begin{proof}
    The expression of the estimator $\widehat{v}_k^{(\text{\rm lin})}$ can be expanded using equation (\ref{eq:projz}):
    $$\widehat{v}_k^{(\text{\rm lin})}=
    \sum_{\ell=1}^L w_\ell \cdot \bigl|\V p_k^\dagger \cdot [\V z]_\ell\bigr|^2=
    \sum_{\ell=1}^L w_\ell \cdot \V p_k^\dagger \cdot [\V z]_\ell\cdot[\V z]_\ell^\dagger\cdot \V p_k
    $$
    By substituting $\M C$ with the expression of the linear multi-channel estimator
    $\widehat{\M C}^{(\text{\rm lin})}$ and $\widehat{v}_k^{(\text{\rm lin})}$ with its
    expansion, the sum of squares in equation (\ref{eq:argmin1}) becomes:
    \begin{multline*}
        \sum_{k=1}^K
    \left(\V p_k^\dagger\cdot\M C\cdot\V p_k-\widehat{v}_k^{(\text{\rm lin})}\right)^2=\\
    \sum_{k=1}^K
    \biggl(\V p_k^\dagger\cdot\left(\sum_{\ell=1}^L w_\ell \cdot [\V z]_\ell
    \cdot [\V z]_\ell^\dagger\right)\cdot\V p_k\\
    -\sum_{\ell=1}^L w_\ell \cdot \V p_k^\dagger \cdot [\V z]_\ell\cdot[\V z]_\ell^\dagger\cdot \V p_k\biggr)^2
    =0
    \end{multline*}
    
    By definition (see equation (\ref{eq:argmin1})), the estimator $\widehat{\M C}^{(\text{\rm \algoname})}$ minimizes this sum of
    squares. When the matrix $\M Q\cdot\M Q^\dagger$ is invertible, this estimator is uniquely
    defined. Since $\widehat{\M C}^{(\text{\rm lin})}$ reaches the lowest possible value for the
    sum of squares, it corresponds to the least squares estimator, i.e., it coincides with
    \algoname's estimator.
\end{proof}

Note that, since $\widehat{\M C}^{(\text{lin})}$ is linear, weights $w_\ell$ can not
depend on the data $\V{z}$. This result applies to filters like spatial multi-looking (all
weights $w_\ell$ are then equal),
convolutions (i.e., shift-invariant weights),
spatially-adaptive filtering with weights driven solely by an external image
(i.e., guided-filtering, see for example \cite{verdoliva2014optical} when only the optical
image is used to compute the weights) or even given by an oracle.
It does not apply, though, to patch-based methods such as \cite{deledalle2014nl} that compute
weights $w_\ell$ based on the similarity between multi-channel SAR patches.

Proposition \ref{prop:equivlin} implies that, provided that the weights $w_\ell$ match for
multi-channel and single-channel filtering, the despeckling result is the same. In a constant
region, using identical weights for all pixels $\ell$ of the region in the estimator
$\widehat{\M C}^{(\text{lin})}$ leads to the maximum lilelihood
estimator, which is asymptotically efficient. \algoname also is asymptotically efficient
in this case and there is no increased variance or bias in the estimation due to processing
only projections. Yet, the cost of providing only single-channel images to the despeckling
algorithm comes from the impossibility to exploit jointly all channels (i.e., use the
co-location of spatial structures in the different channels). A drop of performance is to
be expected for non-linear estimators due to the independent processing of each projection.
The very good restoration capability of modern single-channel despeckling techniques limits
this impact, as shown in section \ref{sec:appli}. We believe that the advantages of \algoname
outweight this drawback, in particular the possibility to perform a self-supervised training,
as shown by the next proposition:

\begin{proposition} \label{prop:ReImindep}
    If the SAR system response is real-valued then the projections $\{\V s_k\}_{k=1..K}$
    defined in equation (\ref{eq:projz}) have statistically independent real and imaginary
    parts. The self-supervised learning framework MERLIN \cite{dalsasso2021if} is then
    applicable.
\end{proposition}
\begin{proof}
In section \ref{sec:muchapro:ppe} we considered an ideal SAR system response. Yet, in practice,
due to zero-padding and spectral apodization, a low-pass filtering correlates the speckle. Let
$\M H\in\mathbb{R}^{N\times N}$ be the linear operator, in the spatial domain, that models this low-pass
response. Assuming that $\M H$ is real-valued amounts to a transfer function with Hermitian
symmetry for shift-invariant systems. Paragraph \ref{sec:muchapro:selfsup} discusses how to
ensure that $\M H$ is real-valued. Let us further assume that this SAR system response is applied identically and separately
to all channels of the ideal multi-channel SAR image $\V z$:
$$\forall \ell,\,[\tilde{\V z}]_\ell=\sum_{i=1}^N \text{H}_{\ell i}[\V z]_i\,,$$
where $[\tilde{\V z}]_\ell\in\mathbb{C}^D$ corresponds to the vector of complex amplitudes in the
actual SAR image (the image when accounting for the SAR system response, with spatially correlated
speckle). The projections $\{\V s_k\}_{k=1..K}$ are computed from the multi-channel image $\tilde{\V z}$:
$$
    \forall \ell,\,[\V s_k]_\ell=\V p_k^\dagger \cdot [\tilde{\V z}]_\ell\,.
$$
Since $\V s_k$ is obtained by applying linear transforms to complex circular Gaussian random
vectors, the real and imaginary parts $\mathbb{R}\text{e}(\V s_k)$ and $\mathbb{I}\text{m}(\V s_k)$
are jointly Gaussian. Therefore, they are independent if and only if
$\text{Cov}\left[\mathbb{R}\text{e}([\V s_k]_\ell),\mathbb{I}\text{m}([\V s_k]_m)\right]=0$ for
all $\ell$ and all $m$. Since the real and imaginary parts are centered, this covariance is equal
to
\begin{multline*}
    \mathbb{E}\left[\mathbb{R}\text{e}([\V s_k]_\ell)\cdot\mathbb{I}\text{m}([\V s_k]_m)\right]=
    \mathbb{E}\biggl[\mathbb{R}\text{e}\biggl(\V p_k^\dagger\sum_{i=1}^N \text{H}_{\ell i}[\V z]_i\biggr)\\
    \cdot
    \mathbb{I}\text{m}\biggl(\V p_k^\dagger\sum_{i=1}^N \text{H}_{m i}[\V z]_i\biggr)\biggr]
\end{multline*}
Since $\text{H}_{\ell i}$ and $\text{H}_{m i}$ are both real, this expectation is equal to
$$\sum_{i=1}^N\sum_{j=1}^N \text{H}_{\ell i}\text{H}_{m j} \mathbb{E}\left[
    \mathbb{R}\text{e}\bigl(\V p_k^\dagger[\V z]_i\bigr)
    \mathbb{I}\text{m}\bigl(\V p_k^\dagger[\V z]_j\bigr)
\right]$$
The $D$-dimensional random vectors $[\V z]_i$ and $[\V z]_j$ are independent when $i\neq j$,
therefore the expectation is zero in that case. Since
$\V p_k^\dagger[\V z]_\ell\sim\mathcal{N}_{\text{c}}([\V v_k]_\ell)$ and $[\V v_k]_\ell$ is
a scalar, the real and imaginary
parts of $\V p_k^\dagger[\V z]_\ell$ are independent and the expectation is also equal to
zero when $i=j$. We have shown that $\mathbb{R}\text{e}(\V s_k)$ and $\mathbb{I}\text{m}(\V s_k)$
are jointly Gaussian and decorrelated, they are therefore independent which is the requirement
to apply MERLIN self-supervised training, as shown in \cite{dalsasso2021if} and further discussed
in section \ref{sec:muchapro:selfsup}.
\end{proof}

\begin{figure}[t]
    \includegraphics[width=\columnwidth]{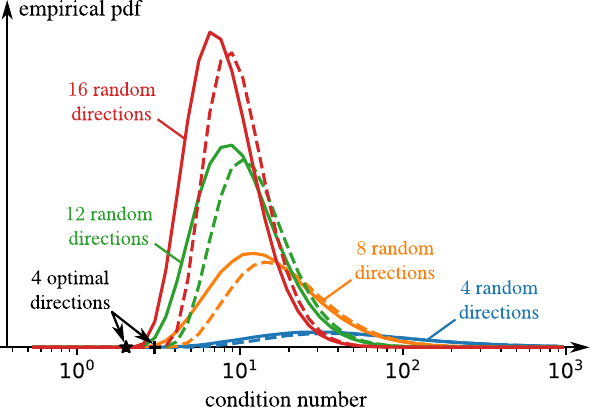}
    \caption{Empirical probability density function of the condition number of matrix $\M Q\cdot\M Q^\dagger$
    for projection directions $\V p_k$ independently drawn according to a standard Gaussian
    distribution. The case where matrix $\M Q$ is built without any constraint on matrix $\M C$
    is displayed with dashed lines. The case where the Hermitian symmetry is leveraged and
    matrix $\M Q$ is defined by (\ref{eq:matQH}) is represented with solid lines. In this
    experiment, $D=2$ and $K$ varies between $D^2$ and $4D^2$.}
    \label{fig:cond}
\end{figure}

\begin{figure}[t]
    \includegraphics[width=\columnwidth]{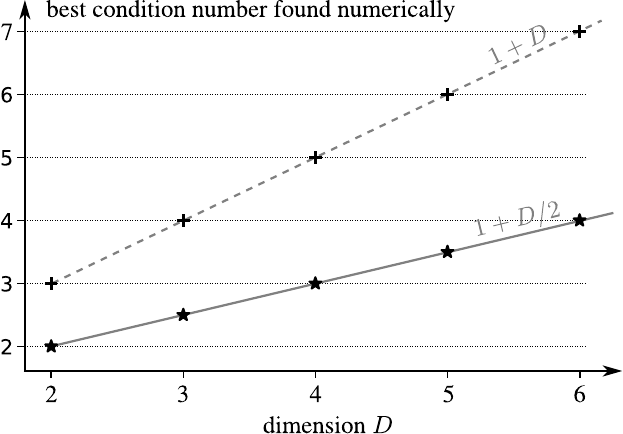}
    \caption{Evolution of the optimal condition number found numerically, as a function of
    the dimension $D$, and expression conjectured from the first 5 points. Solid line: inversion
    using the Hermitian symmetry of the covariance matrix; dashed line: without symmetry
    constraint.}
    \label{fig:cond_trend}
\end{figure}

\subsection{Selecting the projections}
\label{sec:muchapro:selproj}
The directions $\{\V p_k\}_{k=1..K}$ onto which the multi-channel image is projected can be chosen
arbitrarily provided that the covariance matrices can be recovered by least squares inversion.
In particular, a necessary condition for matrix $\M Q\cdot\M Q^\dagger$ to be invertible is that
$K\geq D^2$. The stability of the inversion depends on the condition number of this matrix.
 The collection of directions should therefore be chosen such
that the condition number be as low as possible. 

A simple strategy to set the projection directions could be to use random directions.
As shown in Fig. \ref{fig:cond}, random directions can lead to very poor condition numbers.
The set of random directions leading to the best condition number among one million trials
is still about 10 to 15\% worse than the best directions identified by optimization when
$D=2$ and two orders of magnitude larger when $D=6$.
Exploiting the Hermitian symmetry
offers an improvement of the condition number both for random directions (the distributions
are shifted to the left in Fig. \ref{fig:cond}) and for the optimal directions found by 
optimization using the technique described below (the best condition number drops from 3
to 2 in the 2-dimensional case illustrated in Fig. \ref{fig:cond}, the gap between
the two approaches increases for higher dimensions, as shown in Fig.\ref{fig:cond_trend}).

Optimizing the condition number
of a Gram matrix is a non-smooth and non-convex problem and is thus non-trivial to solve,
see \cite{chen2011minimizing} for an approach based on the Clarke generalized gradient and
an exponential smoothing. 
We tested several optimization strategies to find directions leading to the best possible
condition number. We found experimentally that a quasi-Newton method (L-BFGS
\cite{nocedal1999numerical})
combined with the smoothing approach of Chen {\it et al.} \cite{chen2011minimizing} 
 and Pytorch automatic differentiation was quite efficient when combined with multiple
 random initializations (at least for the small values of $D$
 shown in Fig. \ref{fig:cond_trend}). The smoothing consists of replacing the condition number of a
 Hermitian matrix $\M A$ with the following approximation
 \begin{align}
    \text{cond}_\mu(\M A)=-\frac{\log\left(\sum_{i=1}^n e^{\lambda_i(\M A)/\mu}\right)}
    {\log\left(\sum_{i=1}^n e^{-\lambda_i(\M A)/\mu}\right)}
 \end{align}
 where $\lambda_1(\M A)\geq \cdots \geq \lambda_N(\M A)$ are the eigenvalues of
matrix $\M A$ and $\mu$ is a smoothing parameter.

Interestingly, while increasing the number of random directions helps to improve the
condition number (see the evolution from 4 random directions to 16 random directions in
Fig. \ref{fig:cond}), the optimal condition number found by numerical optimization is
reached even for the lowest number of directions, i.e., $D^2$. This means that the computational
overhead required by processing projections rather than directly handling the multi-channel
images remains
limited and is in fact similar to that of some multi-channel approaches 
(for example, it involves the same number of 2D image denoising operations as a
single iteration of MuLoG algorithm \cite{deledalle2017mulog,deledalle2022speckle}).

The best condition numbers obtained by numerical optimization in the cases $D=1,\,2,\,3\ldots$
are very close to integer or half-integer values, which is a strong indication 
that particular structures for the optimal matrices $\M P$ and $\M Q$ may be found. We
could identify such structures in the case without constraint, the more interesting
case of the Hermitian symmetry seems more intricate, though, and we could not figure out
a simple
expression for the optimal projections. We have to rely on the numerical search to identify
these projections. Fig. \ref{fig:cond_trend} reports the evolution of the best condition
number obtained for different dimensions $D$. We conjecture that these condition numbers
follow a simple affine trend: $1+D$ (formulation without constraint) or $1+D/2$ (formulation
leveraging the Hermitian symmetry). 
Note that, while it is numerically more stable
to use well-conditioned matrices, it is not mandatory to identify the best
overall projection directions for the method to work well.

\begin{figure*}[!t]
    \includegraphics[width=\textwidth]{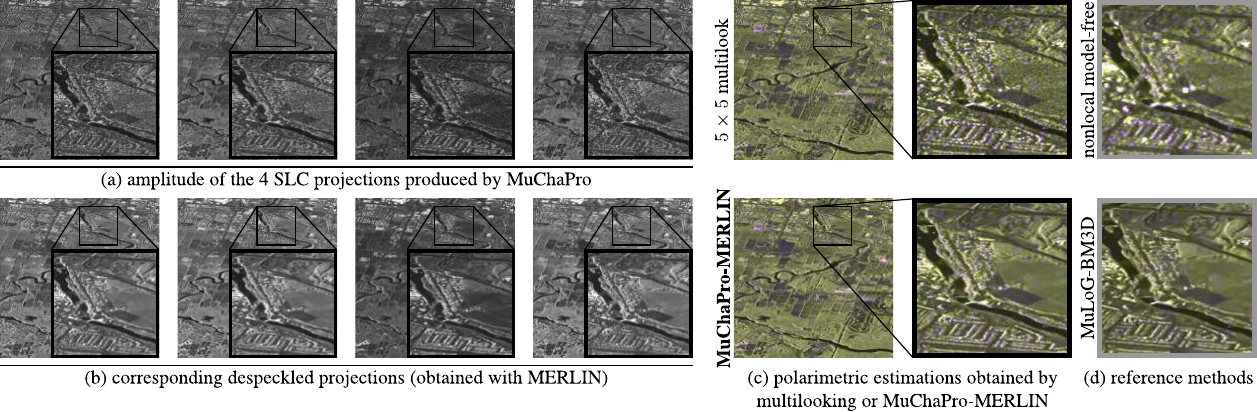}\vspace*{-.5\baselineskip}
    \caption{Application of \algoname to a dual-polarization TerraSAR-X image
    (\copyright ESA/DLR). To estimate the $2\times 2$ polarimetric covariance
    matrices at each pixel, 4 SLC projections are generated in \algoname algorithm
    (shown in (a)). These projections are despeckled independently using a
    deep neural network trained beforehand with MERLIN (despeckling
    results are given in (b)). These despeckled projections are then recombined
    using Algorithm 1, equation (\ref{eq:matQH}), and Algorithm 2 to obtain the
    estimation shown in (c) bottom row, to be compared with a $5\times 5$ multilooking
    in the top row \Change{and the two reference methods in (d): the model-free non-local approach \cite{aghababaei2021nonlocal} and MuLoG-BM3D \cite{deledalle2022speckle}}.}
    \label{fig:dualpol_results}
\end{figure*}

\subsection{Self-supervised learning for multi-channel despeckling}
\label{sec:muchapro:selfsup}
We introduced in \cite{dalsasso2021if}, under the name MERLIN, a self-supervised training
strategy to learn despeckling networks based on the decomposition of SLC images into
real and imaginary parts.
Due to coherence between the channels of a polarimetric or interferometric SAR images, the
real and imaginary values  of multi-channel SAR images are not independent and a direct
extension of the MERLIN framework is not possible for these images. Yet, once projected,
we have shown in Proposition \ref{prop:ReImindep} that the real and imaginary parts of each
projection are statistically independent. It is then possible to apply self-supervised learning
to these images.
In the proof of Proposition \ref{prop:ReImindep} we used the fact that the SAR system response was
real-valued. If this is not the case, for instance, due to a zero-Doppler shift, then correlations
appear and the training does not lead to satisfactory results \cite{dalsasso2023self}: speckle
fluctuations remain because the network could guess the values of the imaginary component based
solely on the values of the real component. SAR images captured by a satellite in stripmap mode
require a simple centering of the Doppler centroid. Other modes like spotlight or TOPS and the
use of a large bandwidth in airborne imaging require additional processing to ensure that the
response is real-valued \cite{dalsasso2023self}. The power spectrum density of the SLC image
is expected to be symmetrical when the response is real. Spectral shifts, demodulation, and
restriction to a symmetrical support can be used to enforce this symmetry.

Once the symmetry of the spectrum has been ensured, a single-channel deep neural network can be
trained by feeding the real (or imaginary) component to the network and evaluating the likelihood
of the reflectivities estimated by the network with respect to the other component (imaginary
or real), as described in \cite{dalsasso2021if}. Since projections are similar irrespective of the
number of channels, a single network can be trained for a given sensor
and then applied to restore polarimetric and interferometric acquisitions from the same sensor.
The ability to vary the number of channels without the need to retrain the network combined with
the possibility to apply self-supervised learning represent very appealing features of the
\algoname framework.

\begin{figure*}[!t]
    \includegraphics[width=\textwidth]{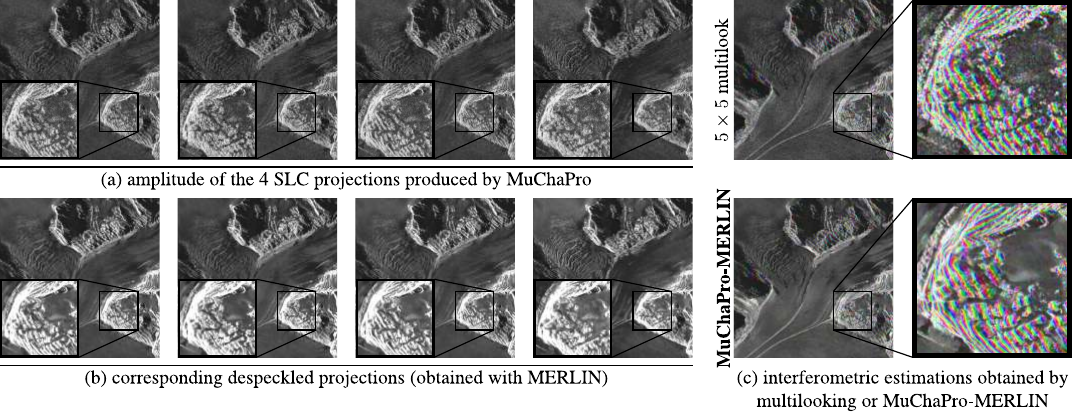}\vspace*{-.5\baselineskip}
    \caption{Application of \algoname to an interferometric pair of TerraSAR-X images
    (\copyright ESA/DLR) of the confluence of glaciers named \emph{Konkordiaplatz}, close
    to the Jungfrau summit in Switzerland.
    As in Fig. \ref{fig:dualpol_results}, we show in (a) the noisy projections and in (b)
    the despeckling results. Combining these restored projections leads to MuChaPro
    estimates shown in (c). The use of MERLIN despeckling network gives a better preservation
    of interferometric fringes than by multilooking, with an improved spatial smoothing.
    In this color composition, the hue corresponds to the interferometric phase, the
    saturation to the coherence, and the value is given by the restored reflectivity of
    the reference date.}
    \label{fig:Jungfrau_results}
\end{figure*}

\section{Application to PolSAR, InSAR, and multi-baseline InSAR}
\label{sec:appli}

In this section, we illustrate the application of \algoname to polarimetric, and interferometric
SAR images. \Change{Note that \algoname performs equally well, in the single-channel case, as any despeckling technique provided that the same technique is used within \algoname: projections in the single-channel case boil down to multiplications by a scalar value. In what follows, illustrations and comparisons therefore only cover multi-channel cases.}

We first illustrate the applications of the method in polarimetry and interferometry.
Figure \ref{fig:dualpol_results} illustrates the application of \algoname
on a TerraSAR-X dual-polarization image
(area north of Belleville, Ontario, Canada, captured on June 3, 2016 \copyright DLR/ESA).
Projections along 4 optimal directions (optimal in the sense of the condition number, as
discussed in paragraph \ref{sec:muchapro:selproj})
are shown in Fig.\ref{fig:dualpol_results}(a). Application of the MERLIN despeckling
algorithm\footnote{we used the pretrained network made available by Telecom Paris and
Hi! PARIS at \url{https://github.com/hi-paris/deepdespeckling/}, after a resampling step
to obtain the same transfer function as single-polarization TerraSAR-X images used during
the training phase of the network} leads to the restored images in Fig.\ref{fig:dualpol_results}(a).
Application of the inversion formulae (\ref{eq:muchaproformula}) and (\ref{eq:matQH})
give the restored polarimetric image shown in Fig.\ref{fig:dualpol_results}(c), bottom row.
A qualitative comparison with the multilooked polarimetric estimation shows that polarimetric
behaviors are in agreement and that the spatial resolution is better preserved when applying
\algoname with MERLIN. Note that we have checked numerically that applying multilooking to the polarimetric
image gives the exact same result as multilooking the projections first and then inverting
the projections to reconstruct the covariance matrices, as stated in Proposition \ref{prop:equivlin}. \Change{Restorations obtained with two polarimetric despeckling techniques are shown in Fig.\ref{fig:dualpol_results}(d): with the model-free nonlocal method \cite{aghababaei2021nonlocal} (first row) and the plug-in ADMM method MuLoG \cite{deledalle2022speckle}. \algoname achieves a better tradeoff between the reduction of speckle fluctuations and the preservation of textures and sharp structures. When using the deep-neural network MERLIN to perform the single-channel despeckling, MuChaPro is also fast: processing the dual-polarization image of size $2336\times 2336$ pixels takes less than 3 seconds, to be compared with the model-free nonlocal method (approx. 8 hours) and MuLoG (approx. 4 minutes).}

In figure \ref{fig:Jungfrau_results}, we illustrate the application of \algoname to an 
interferometric couple of TerraSAR-X Stripmap images captured close to the Jungfrau summit in 
Switzerland on July 4, 2020, and July 15, 2020 (\copyright DLR/ESA).
The same optimal projection directions are used as in Fig.\ref{fig:dualpol_results}. The 
combination of the two images leads to visible fringes in some of the projections, see 
\ref{fig:Jungfrau_results}(a) and \ref{fig:Jungfrau_results}(b).
After inversion, the interferometric fringes are recovered. In \ref{fig:Jungfrau_results}(c), 
we show the recovered interferogram in the form of a color image with the hue encoding the 
interferometric phase, the saturation the coherence, and the value the reflectivity of the 
reference image. In this mountainous area at the confluence of several glaciers, there are 
low-coherence rapidly moving areas and higher-coherence areas where topographic fringes are
recovered. The qualitative comparison of the interferogram obtained by $5\times 5$ multilooking
and with \algoname shows that the fringes have similar locations and orientations but that 
thin fringes are better recovered by \algoname in conjunction with MERLIN and appear much smoother.
Note that to slightly improve the interferogram estimation, we performed a despeckling of each image of
the interferometric pair and substituted the reflectivities recovered in MuChaPro's 
interferogram with those estimates. 

\begin{figure*}[t]
    \includegraphics[width=\textwidth]{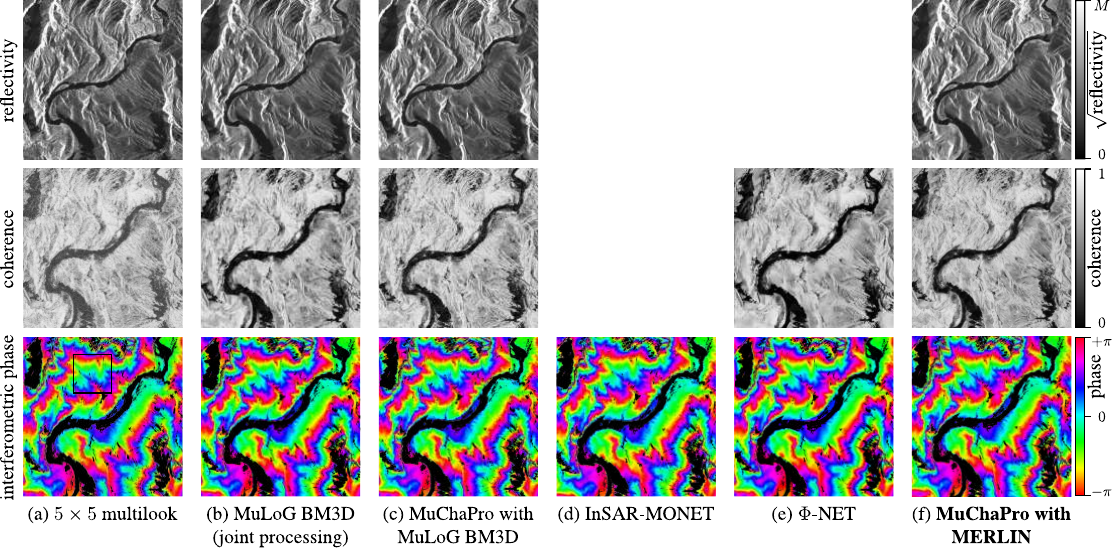}
    \caption{Computation of an interferogram from a pair of TerraSAR-X images over the Colorado
    river \copyright Airbus Defence and Space. Images are best viewed on-screen with a strong magnification.
    Fig.\ref{fig:Colorado_results_zoom} gives a zoomed-in version of the area depicted
    by a black square in column (a), last row.}
    \label{fig:Colorado_results}
\end{figure*}

\begin{figure*}[t]
    \includegraphics[width=\textwidth]{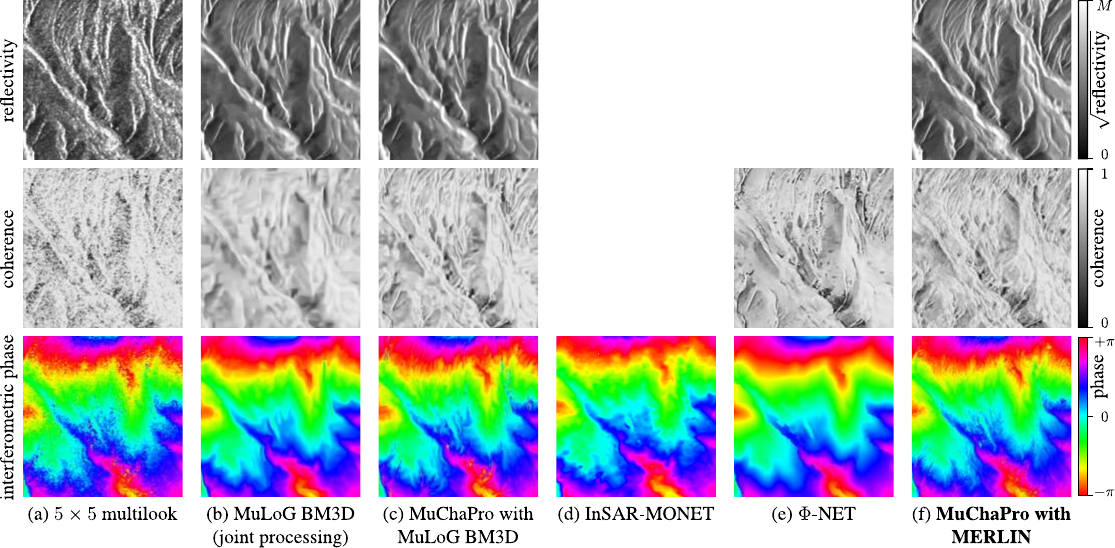}
    \caption{Close-up view of the area delimited by a black rectangle in 
     Fig.\ref{fig:Colorado_results}(a) (SAR image \copyright Airbus Defence and Space).}
    \label{fig:Colorado_results_zoom}
\end{figure*}

\medskip
In figure \ref{fig:Colorado_results}, we consider another TerraSAR-X Stripmap interferogram over an area of the Grand Canyon, Arizona, USA (\copyright Airbus Defence and Space), where the coherence is much larger. We compare the interferograms 
recovered by several methods: (a) multilooking, (b) multi-channel estimation with MuLoG 
\cite{deledalle2017mulog,deledalle2022speckle}, (c) the application of MuLoG to the projections
produced by \algoname followed by the inversion of the projections to recover the 
interferogram, (d) the deep-learning method InSAR-MONET \cite{vitale2022insar} (that only
produces an estimate of the interferometric phase), (e) the deep-learning method $\Phi$-NET 
\cite{sica2021net} (that gives both an estimation of the phase and coherence), and 
finally (f) \algoname with MERLIN as a single-channel image despeckling method.
Before processing the images with MuLoG, we performed a downsampling by a factor 2 to
reduce the spatial correlations of the speckle. The images have been resampled to remove 
all zero-padding and reduce speckle correlations before applying InSAR-MONET and $\Phi$-NET,
we found this preprocessing step to improve the results with these networks.
When applying MERLIN, we did not resample the images since the network was already trained 
to handle the spatially correlated speckle in TerraSAR-X single-polarization Stripmap images.
Figure \ref{fig:Colorado_results_zoom} offers zoomed-in views to analyze the details 
recovered by each method. The interferometric phases are in agreement with all methods
but are much more noisy with multilooking. Interestingly, \algoname seems to better 
transfer the fine spatial details observed in the reflectivity to the coherence and 
interferometric phase channels. 

\begin{figure*}[t]
    \includegraphics[width=\textwidth]{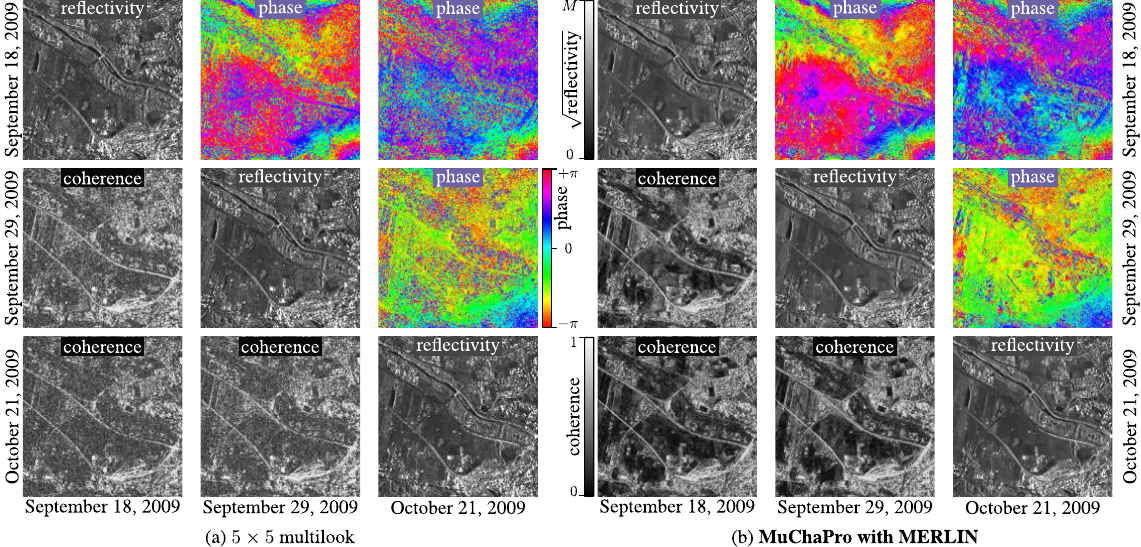}
    \caption{Application of MuChaPro to multi-baseline interferometry: TerraSAR-X images
    captured over Domancy, France, at 3 different dates are combined in a multi-baseline
    interferometric covariance matrix.}
    \label{fig:Domancy_results}
\end{figure*}

\medskip
Finally, we show in Figure \ref{fig:Domancy_results} that \algoname can also be applied to
multibaseline interferometry. We processed 3 TerraSAR-X Stripmap images of the city of Domancy, France
 (\copyright DLR) captured in September and October 2009. Parts of the images covered by 
 vegetation display a low coherence. This coherence tends to be over-estimated by multi-looking.
 The interferometric phase estimated by \algoname and multilooked are in agreement at a 
 large scale. It is less noisy, in particular in low coherence areas corresponding to fields on the
 lower-diagonal part of the image.

 
\section{Conclusion}
This paper introduced a new generic framework to perform multi-channel SAR despeckling using only
single-channel despeckling algorithms. The idea is to produce a set of single-channel SAR images
by projecting the multi-channel SAR image onto different directions. Once despeckled, these images
give access to projections of the underlying covariance matrices which can be recovered thanks to
a linear least-squares estimator.

This approach is equivalent to multi-channel processing when linear filters are applied. It also makes
the self-supervised learning strategy MERLIN possible. To the best of our knowledge, this is the
first time that a self-supervised training strategy has been proposed for interferometric and polarimetric
data. A key advantage of the method is that, once a network is trained for a given sensor, it can
be readily applied to multi-channel images of the same sensor whatever their dimensionality: dual-pol,
full-pol, InSAR, multi-baseline InSAR, \ldots

This flexibility comes at a price: by despeckling separately single-channel images, \algoname does
not benefit from the redundancy present in multi-channel SAR imaging when restoring weakly contrasted spatial
structures. The ability to train networks in a self-supervised way, a possibility that is not (yet)
offered by joint despeckling techniques, tends to mitigate this drawback, though.

The application of \algoname over different TerraSAR-X images in polarimetric or 
 interferometric configurations confirms the potential of the proposed approach to recover 
 the information with a high spatial resolution. Further work is necessary to perform a
 more in-depth analysis of the performance and limits of this approach.

\section*{Acknowledgment}

The authors would like to thank the French National Research Agency (ANR) and the Direction Générale de l’Armement (DGA) under ASTRAL project ANR-21-ASTR-0011 for funding this research.
They are grateful to ESA for making available TerraSAR-X images through the 
TerraSAR-X ESA archive collection, disseminated under ESA's Third Party Missions.
The SAR image of Domancy was obtained from the DLR under project LAN-1706. The SAR image of the Colorado is part of sample data freely distributed by Airbus Industry.

\ifCLASSOPTIONcaptionsoff
  \newpage
\fi

\bibliographystyle{IEEEtran}
\bibliography{IEEEabrv,refs}

\end{document}